\documentclass{article}

\usepackage[preprint,nonatbib]{neurips_2023}

\usepackage{booktabs}
\usepackage[utf8]{inputenc} 
\usepackage[T1]{fontenc}    
\usepackage{url}            
\usepackage{booktabs}       
\usepackage{amsfonts}       
\usepackage{nicefrac}       
\usepackage{microtype}      
\usepackage{xcolor}         
\usepackage{bm}
\usepackage{amsmath,amssymb,amsthm}
\usepackage{tikz}
\usepackage{float}
\usetikzlibrary{calc}
\usepackage{wrapfig}
\usepackage{caption}
\usepackage[normalem]{ulem}
\useunder{\uline}{\ul}{}

\usepackage{graphicx}
\usepackage{csquotes}
\usepackage{cancel}
\usepackage{algorithm}
\usepackage{algorithmic}
\usepackage{enumitem}
\usepackage{pifont}
\usepackage{thmtools}
\usepackage{cleveref}

\def\v{{\bf v}}

\def\P{{\bf P}}
\def\X{{\bf X}}
\def\Y{{\bf Y}}
\def\x{{\bf x}}

\def\y{{\bf y}}
\def\Xbb{{\mathbb{X}}}

\def\Ybb{{\mathbb{Y}}}
\def\Z{{\bf Z}}

\def\gg{{\gamma}}
\def\ga{{\boldsymbol \gamma}}
\def\R{{\mathbb{R}}}

\def\G{\gamma}

\def\GG{{\bm \G}}

\def\GGs{\GG^s}

\def\GGv{\GG^v}
\def\w{{\textbf w}}
\def\v{{\textbf v}}

\def\R{{\mathbb{R}}}

\def\Cbf{{\mathbf C}}
\def\D{{\mathbf D}}

\newcommand{\COOT}{\text{COOT}}
\newcommand{\AGW}{\text{AGW}}
\newcommand{\GW}{\text{GW}}

\newcommand{\ie}{\textit{i.e.}}

\newtheorem{theorem}{Theorem}

\newtheorem{proposition}{Proposition}
\newtheorem{corollary}{Corollary}
\newtheorem{definition}{Definition}

\theoremstyle{remark}

\title{Revisiting invariances and introducing priors in Gromov-Wasserstein distances}

\author{
  Pinar Demetci\\
  Department of Computer Science\\
  Center for Computational Molecular Biology\\
  Brown University\\
  Providence, RI 02912\\
  \texttt{pinar\textunderscore demetci@brown.edu}\\
  \And
  Quang Huy Tran\\
  Univ. Bretagne-Sud, IRISA \\
  Vannes, France 56000\\
  CMAP, Ecole Polytechnique, IP Paris\\
  Palaiseau, France 91120\\
  \texttt{quang-huy.tran@univ-ubs.fr} \\
  \AND
  Ievgen Redko \thanks{Co-corresponding authors}\\
  Huawei Technologies \\
  Univ. Lyon, UJM-Saint-Etienne \\
  CNRS, UMR 5516\\ 
  Saint-Étienne, France 42100 \\
  \texttt{ievgen.redko@huawei.com} \\
  \And
  Ritambhara Singh$^{*}$\\
  Department of Computer Science\\
  Center for Computational Molecular Biology\\
  Brown University\\
  Providence, RI 02912\\
  \texttt{ritambhara@brown.edu}
}

\begin{document}
\maketitle
\begin{abstract}


    
    Gromov-Wasserstein distance has found many applications in machine learning due to its ability to compare measures across metric spaces and its invariance to isometric transformations. However, in certain applications, this invariance property can be too flexible, thus undesirable. Moreover, the Gromov-Wasserstein distance solely considers pairwise sample similarities in input datasets, disregarding the raw feature representations. We propose a new optimal transport-based distance, called Augmented Gromov-Wasserstein, that allows for some control over the level of rigidity to transformations. It also incorporates feature alignments, enabling us to better leverage prior knowledge on the input data for improved performance. We present theoretical insights into the proposed metric. We then demonstrate its usefulness for single-cell multi-omic alignment tasks and a transfer learning scenario in machine learning.
\end{abstract}
\let\clearpage\relax

\section{Introduction}
Optimal transport (OT) theory provides a fundamental tool for comparing and aligning probability measures omnipresent in machine learning (ML) tasks. Following the least effort principle, OT and its associated metrics offer many attractive properties that other divergences, such as the popular Kullback-Leibler or Jensen-Shannon divergences, lack. For instance, OT  borrows key geometric properties of the underlying ``ground'' space on which the distributions are defined \cite{Villani} and enjoys non-vanishing gradients in case of measures having disjoint support \cite{arjovsky17a}. OT theory has also been extended to a much more challenging case of probability measures supported on different metric-measure spaces. In this scenario,  Gromov-Wasserstein (GW) distance \cite{memoli_gw} seeks an optimal matching between points in the supports of the considered distributions by using the information about the distortion of intra-domain distances after such matching. Since its proposal by Memoli \cite{memoli_gw} and further extensions by Peyre \textit{et al} \cite{peyre2016gromov}, GW has been successfully used in a wide range of applications, including computational biology \cite{ WaddingtonOT,Pamona, UniPort, SpaOTsc, Demetci2020,Demetci2022, PASTE}, generative modeling \cite{bunne_gan}, and reinforcement learning \cite{GW-VAE,FickingerC0A22}. 

\paragraph{Limitations of prior work} Successful applications of GW distance are often attributed to its invariance to distance-preserving transformations (also called isometries) of the input domains. Since GW considers only intra-domain distances, it is naturally invariant to any transformation that does not change them. 
While these invariances can be a blessing in many applications, for example, comparing graphs with the unknown ordering of nodes, they may also become a curse when one must choose the ``right" isometry from those for which GW attains the same value. How one would break such ties while keeping the attractive properties of the GW distance? To the best of our knowledge, there are no prior works addressing this question. 

Additionally, GW distances are often used in tasks where one may have some \textit{a priori} knowledge about the mapping between the two considered spaces. For example, in single-cell applications, mapping a group of cells in similar tissues across species helps understand evolutionarily conserved and diverged cell types and functions \cite{kriebel2022uinmf}. This cross-species cell mapping, when performed using OT, may benefit from the knowledge about an overlapping set of orthologous genes \footnote {Genes that diverge after a speciation event from a single gene in a common ancestor, but their main functions, as well as the genetic sequences, are largely conserved across the different species.}. GW formulation does not offer any straightforward way of incorporating this knowledge, which may lead to its suboptimal performance in the above-mentioned tasks. 

\paragraph{Our contributions} In this paper, we aim to address the drawbacks of the GW distance mentioned above. We propose to augment GW distance with an additional loss term that allows tightening its invariances and incorporating prior knowledge on how the two input spaces should be compared. Overall, our contributions can be summarized as follows:
\begin{enumerate}
    \item We present a new metric on the space of probability measures that allows for better control over the isometric transformations of the GW distance;
    \item We provide some theoretical analysis of the properties of the proposed distance, as well as an experimental example that illustrates its unique features vividly;
    \item We empirically demonstrate that such a new metric is more efficient than previously proposed cross-domain OT distances in several single-cell data integration tasks and its generalizability to the ML domain.
\end{enumerate}
The paper is organized as follows. Section 2 presents key notions from the OT theory utilized in the rest of the paper. Section 3 presents our proposed distance and analyzes its theoretical properties. In Section 4, we present several empirical studies for the single-cell alignment task and demonstrate the applicability of our metric in another ML domain. We conclude our paper in Section 5 by discussing limitations and potential future work.

\paragraph{Notations} In what follows, we denote by $\Delta_{n}=\{w \in (\mathbb{R}_{+})^{n}:\ \sum_{i=1}^{n} w_{i}=1\}$ the simplex histogram with $n$ bins. We use $\otimes$ for tensor-matrix multiplication, \ie, for a tensor $L =(L_{i,j,k,l})_{i,j,k,l}$ and a matrix $B = (B_{i,j})_{i,j}$, the tensor-matrix multiplication $L \otimes B$ is the matrix $(\sum_{k,l} L_{i,j,k,l} B_{k,l})_{i,j}$. We use $\langle \cdot, \cdot \rangle$ for the matrix scalar product associated with the Frobenius norm $\|\cdot\|_{F}$. Finally, we write $\bm{1}_d \in \mathbb{R}^d$ for a $d$-dimensional vector of ones.
We use the terms ``coupling matrix'', ``transport plan'' and ``correspondence matrix'' interchangeably.
A point in the space can also be called ``an example'' or ``a sample''. Given an integer $n \geq 1$, denote $[n] := \{ 1, ..., n\}$.


\section{Preliminary knowledge}
In this section, we briefly present the necessary background knowledge required to understand the rest of this paper. This includes introducing the Kantorovich formulation of the OT problem and two OT-based distances proposed to match samples across incomparable spaces.
\paragraph{Kantrovich OT and Wasserstein distance} Let $\X \in \R^{n\times d}$ and $\Y\in \R^{m\times d}$ be two input matrices,  $\Cbf_{ij} = c(\x_i, \y_j)$ be a cost (or ground) matrix defined using some lower semi-continuous cost function $c: \R^d \times \R^d \rightarrow \R_{\geq 0}$. Given two discrete probability measures $\mu \in \Delta_n$ and $\nu\in \Delta_m$, Kantorovich formulation of OT seeks a coupling $\gg$ minimizing the following quantity:
\begin{equation}
  W_\Cbf(\mu,\nu) = \min_{\ga \in \Pi(\mu,\nu)} \langle \Cbf, \GG \rangle,
  \label{eq:wasserstein}
 \end{equation}
where $\Pi(\mu,\nu)$ is the space of probability distributions over $\mathbb{R}^2$ with marginals $\mu$ and $\nu$. Such an optimization problem defines a proper metric on the space of probability distributions called the Wasserstein distance.


\paragraph{Gromov-Wasserstein distance}
When samples of input matrices live in different spaces, \ie, $\X \in \R^{n\times d}$ and $\Y\in \R^{m\times d'}$ with $d \neq d'$, they become incomparable, \textit{i.e.} it is no longer possible to define a cost function $c$ as the distance between two points across the available samples. In this case, one cannot use the Wasserstein distance defined above. To circumvent this limitation, one can use the Gromov-Wasserstein (GW) distance \cite{memoli_gw}, defined as follows:
\begin{align}
\label{eq:gw}
    \text{GW}(\X, \Y, \mu, \nu, d_X, d_Y) := \min_{\GG \in \Pi(\mu, \nu)} \mathcal{L}_{GW} (\GG)
\end{align}
where
\begin{align}
\label{eq:gw_func}
    \mathcal{L}_{\text{GW}} (\GG) := \sum_{i,j,k,l} \big( d_X(x_i,x_k) - d_Y(y_j,y_l) \big)^2 \gg_{i,j}\gg_{k,l} = \langle L(\D_X, \D_Y) \otimes \GG, \GG \rangle.
\end{align}
Here, the tensor $L(\D_X, \D_Y)$ is defined by $\big( L(\D_X, \D_Y) \big)_{i,j,k,l} = \big( d_X(x_i,x_k) - d_Y(y_j,y_l) \big)^2$, where $(x_i, x_k) \in \R^2$ and $(y_j, y_k) \in \R^2$ are tuples of 1D coordinates of samples in $\X$ and $\Y$ and $d_X$ and $d_Y$ are proper metrics so that $(\D_X)_{ik} = d_X(x_i,x_k)$ and $(\D_Y)_{jl} = d_Y(y_j,y_l)$.



\paragraph{CO-Optimal transport}
Redko \textit{et al.} \cite{COOT} introduced an alternative to GW distance, termed CO-Optimal transport (COOT). Rather than relying on the intra-domain distance matrices $\D_X$ and $\D_Y$ like GW does, COOT instead takes into account the feature information (\textit{i.e.} the coordinates of the samples) and jointly learns two couplings, corresponding to the sample and feature alignments. More precisely, COOT distance between two input matrices $\X$ and $\Y$ is defined by
\begin{equation}
  \label{eq:co-optimal-transport}
 \begin{split}
     \COOT(\X, \Y, \mu, \nu, \mu', \nu') := \min_{\GGs \in\Pi(\mu,\nu),\GGv\in\Pi(\mu',\nu')} \mathcal{L}_{\COOT} (\GGs, \GGv)
    \end{split}
 \end{equation}  
where 
$$\mathcal{L}_\COOT (\GGs, \GGv) :=  \sum_{i,j,k,l} L(x_{ik},y_{jl})\GGs_{i,j}\GGv_{k,l} = \langle L(\X,\Y) \otimes \GGv,  \GGs \rangle$$ 
with $\mu' \in \Delta_d$ and $\nu' \in \Delta_{d'}$ being empirical distributions associated with the features (columns) of $\X$ and $\Y$. In what follows, we consider $L(x_{ik}, y_{jl}) = (x_{ik} - y_{jl})^2$ and write simply $\GW(\X, \Y)$ and $\COOT(\X, \Y)$ when $\mu, \nu, \mu', \nu'$ are uniform and when the choice of $d_X$ and $d_Y$ is of no importance.



\section{Our contributions}
Here, we first start by outlining the motivation for our proposed divergence, highlighting the different properties of GW distance and COOT. Then, we detail our new metric that interpolates between the two, followed by a theoretical study of its properties.
\subsection{Motivation}
\paragraph{Invariances of GW distance} The invariances encoded by GW distance are characterized by the condition where 
$\text{GW}(\X, \Y) = 0$. In the discrete setting, this is equivalent to the existence of a measure-preserving isometry $f: \R^d \to \R^{d'}$, that is $f_\# \mu = \nu$ and $d_X(\cdot, \cdot) = d_{f(X)}(\cdot, \cdot)$. In particular, this also implies that $\X$ and $\Y$ have the same cardinality.

In other words, GW distance remains unchanged under any isometric transformation of input data. This favorable property has contributed much to the success and popularity of GW distance, where in many applications, the isometries naturally appear. However, since there are infinitely many isometries, not all are equally desirable. For instance, a rotation of the digit $6$ seen as a discrete measure can either lead to its slight variation for small angles or to a digit 9 when the angle is close to 180 degrees. In both cases, however, the GW distance remains unchanged, although it is clearly detrimental for telling the two distinct objects apart. 

\paragraph{Invariances of COOT} Unlike GW, COOT has fewer degrees of freedom in terms of invariance to global isometric transformations as it is limited to permutations of rows and columns of the two matrices, and not all isometric transformations can be achieved via such permutations. As an example, Appendix Figure 1 shows the empirical effect of the sign change and image rotation in a handwritten digit matching task, where GW is invariant to such transformations, while COOT is not. 
Additionally, as shown in \cite{COOT}, COOT distance vanishes precisely when there exist two bijections $\sigma_s: [n] \to [m]$ and $\sigma_f: [d] \to [d']$ such that
\begin{itemize}
    \item $(\sigma_s)_{\#} \mu = \nu$ and $(\sigma_f)_{\#} \mu' = \nu'$ (which also imply that $n = m, d = d'$).
    \item $x_{ij} = y_{\sigma_s(i)\sigma_f(j)}$, for every $(i, j) \in [n] \times [d]$.
\end{itemize}
Therefore, COOT is strictly positive for any two datasets of different sizes both in terms of features and samples, making it much more restrictive than GW. It thus provides a fine-grained control when comparing complex objects, yet it lacks the robustness of GW to frequently encountered transformations between the two datasets. 

\subsection{Augmented Gromov-Wasserstein distance}
\begin{figure}[!t]
\centering
\includegraphics[width=\linewidth]{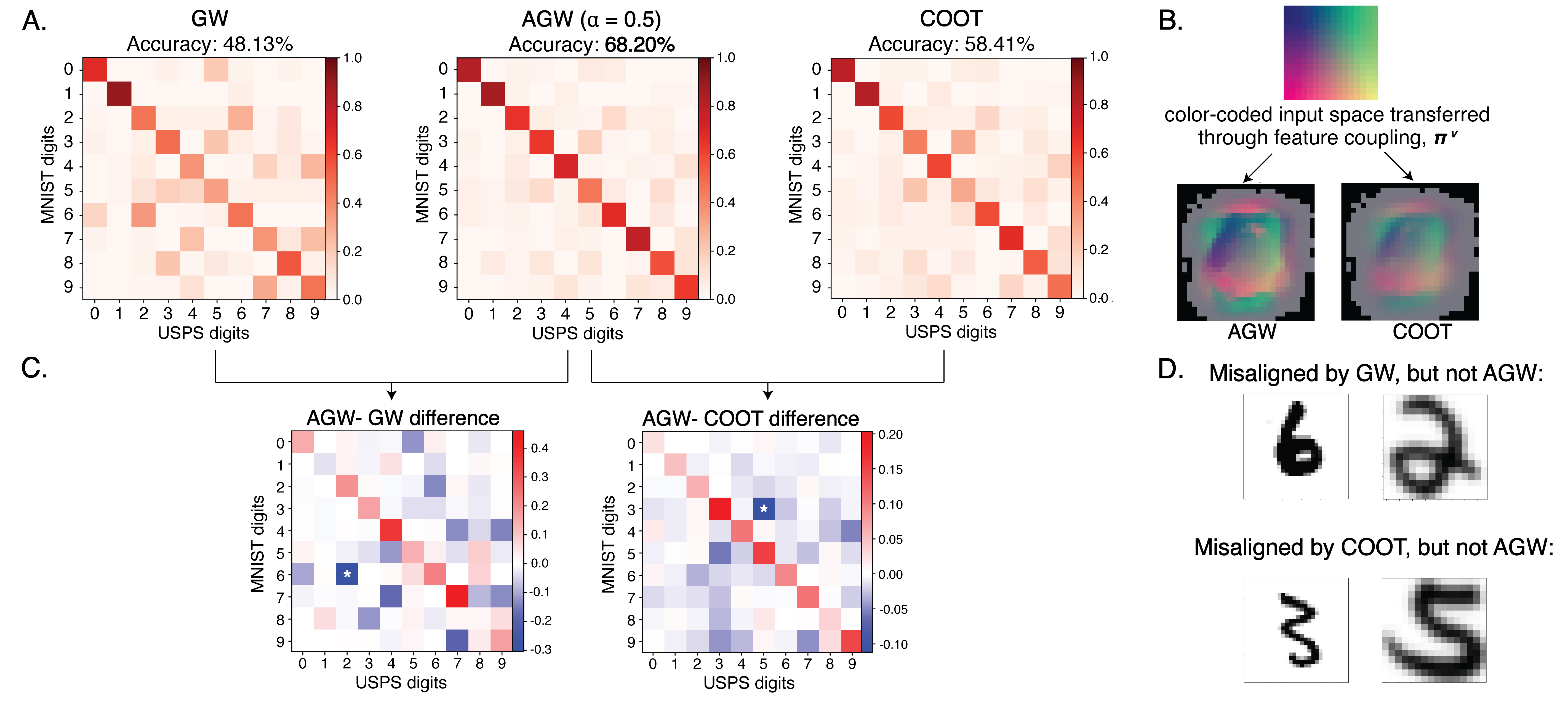}
\caption{Aligning digits from MNIST and USPS datasets. \textbf{(A)} Confusion matrices of GW, AGW with $\alpha=0.5$ and COOT; \textbf{(B)} Feature coupling $\GGv$ of AGW compared to COOT; \textbf{(C)} Difference between the sample couplings obtained with AGW and GW or COOT; \textbf{(D)} Illustration of a case from (C) where GW's and COOT's invariances are detrimental for obtaining a meaningful comparison, while AGW remains informative.}
\label{fig:mnist}
\end{figure}
Given the above discussion on the invariances of COOT and GW distance, it appears natural to propose a novel distance, that we term \textbf{augmented GW} (AGW), interpolating between them as follows:
\begin{align}
\label{eq:scootr}
\text{$\AGW$}_{\alpha}(\X, \Y) &:=  \min_{ \scriptsize{\begin{matrix}\GGs \in\Pi(\mu,\nu),\\ \GGv \in \Pi(\mu',\nu')\end{matrix}}} \alpha \mathcal{L}_{\text{GW}} (\GGs) + (1-\alpha) \mathcal{L}_{\COOT} (\GGs, \GGv) \\
&=\min_{ \scriptsize{\begin{matrix}\GGs \in\Pi(\mu,\nu),\\ \GGv \in \Pi(\mu',\nu')\end{matrix}}} \langle \alpha L(\D_X,\D_Y)\otimes \GGs + (1-\alpha) L(\X,\Y) \otimes \GGv,  \GGs \rangle.
\end{align}
One may see that the $\AGW$ problem always admits a solution. Indeed, as the objective function is continuous and the sets of admissible couplings are compact, the existence of minimum and minimizer is then guaranteed.

Our proposed interpolation between COOT and GW distance offers several important benefits. First, the loss term in COOT ensures that $\AGW$ will take different values for any two isometries whenever $d \neq d'$. Intuitively, in such a case, $\AGW$'s value will depend on how ``far'' a given isometry is from a permutation of rows and columns of the input matrices. Thus, we restrict a very broad class of (infinitely many) transformations that GW cannot distinguish and tell them apart by assessing whether or not they can be approximately obtained by simply swapping 1D elements in input matrices. 

Second, combining COOT and GW distance allows us to effectively influence the optimization of $\GGs$ by introducing priors on feature matchings through $\GGv$ and vice versa. This can be achieved by penalizing the costs of matching certain features in the COOT loss term to force the optimization of $\GGs$ and take it into account. These two key properties explain our choice of calling it ``augmented": on the one hand, we equip GW distance with an ability to provide more fine-grained comparisons between objects, while on the other hand, we incorporate into it a possibility of guiding the matching using available prior knowledge on the feature level. 

\paragraph{Illustration} We illustrate $\AGW$ on a problem of aligning handwritten digits from MNIST dataset (28$\times$28 pixels) with those from USPS dataset (16$\times$16 pixels) in Figure \ref{fig:mnist}, where AGW with $\alpha=0.5$ outperforms both GW and COOT in alignment accuracy (Panel A). When we investigate the digit pairs that benefit the most from the interpolation in AGW, we notice that misalignment between 6 and 2 in Gromov-Wasserstein OT, and misalignment between 3 and 5 in COOT improves the most (Panel C, highlighted by white asterisks). Panel D visualizes examples of digit pairs misaligned \footnote{Here, we define ``aligned pairs'' as pairs of digits with the highest coupling probabilities} by GW distance or COOT but correctly aligned with their own digits by AGW. We observe from these examples that 6-2 misalignment by GW distance is likely due to the fact that one is close to the reflection of the other across the y-axis. Similarly, COOT confuses 3 and 5 as one can easily obtain 3 from 5 by a local pixel permutation of the upper half of the images. Panel B visualizes the feature couplings obtained by AGW (on left) and COOT (on right). The feature coupling by COOT confirms that COOT allows for a reflection across the y-axis on the upper half of the image, but not on the lower half. With the interpolation in AGW, both of these misalignment cases improve, likely because (1) the correct feature alignments in the lower half of the images prevent 6 and 2 from being matched to each other and (2) GW distance is non-zero for 5-3 matches since GW will only be invariant to global isometries. In Appendix, we also show that providing supervision on feature alignments to restrict local reflections further improves AGW's performance.


\paragraph{Optimization} For simplicity, let us suppose $n = m$ and $d = d'$. Thanks to the squared loss in the GW and COOT loss terms, the computational trick in \cite{peyre2016gromov} can be applied, which helps reduce the overall complexity of $\AGW$ from $O(n^4 + n^2 d^2)$ to $O(n^3 + dn^2 + nd^2)$. To solve the $\AGW$ problem, we use the block coordinate descent (BCD) algorithm, where one alternatively fixes one coupling and minimizes with respect to the other. Thus, each iteration consists in solving two optimal transport problems. To further accelerate the optimization, entropic regularization can be used \cite{cuturi:2013} on either $\GGs$, $\GGv$, or both. Details of the algorithm can be found in \Cref{alg:bcd}. In practice, we use POT package \cite{flamary2021pot}, which contains built-in functions to solve the OT, GW and COOT problems.

\begin{algorithm}[!t]
    \caption{BCD algorithm to solve $\AGW$ \label{alg:bcd}}
    \begin{algorithmic}[t]
      \STATE Initialize $\GGs$ and $\GGv$
      \REPEAT
      \STATE Calculate $L_v = L(\X, \Y) \otimes \GGs$.
      \STATE For fixed $\GGs$, solve the OT problem: $\GGv \in \arg\min_{\GG \in \Pi(\mu', \nu')} \langle L_v, \GG \rangle$.
      \STATE Calculate $L_s = L(\X, \Y) \otimes \GGv$.
      \STATE For fixed $\GGv$, solve the AGW problem: $\GGs \in \arg\min_{\GG \in \Pi(\mu, \nu)} \alpha \mathcal{L}_{\text{GW}} (\GGs) + (1-\alpha) \langle L_s, \GGs \rangle$.
      \UNTIL{convergence}
\end{algorithmic}
\end{algorithm}

\subsection{Theoretical analysis}

Intuitively, given the structure of the objective function, we expect that $\AGW$ should share similar properties with GW and COOT, namely the existence of a minimizer, the interpolation between GW distance and COOT when the interpolation parameter varies, and the relaxed triangle inequality (since COOT and GW distance are both metrics). The following result summarizes these basic properties, and their proofs are in Appendix Section 1.
\begin{proposition}
\label{prop:basic_prop}
For every $\alpha \in [0, 1]$,
\begin{enumerate}
    \item Given two input matrices $\X$ and $\Y$, when $\alpha \to 0$ (or $1$), one has $\text{$\AGW$}_{\alpha}(\X, \Y) \to \text{COOT}(\X, \Y)$ (or $\text{GW}(\X, \Y)$).

    \item $\AGW$ satisfies the relaxed triangle inequality: for any input matrices $\X, \Y, \Z$, one has $\text{$\AGW$}_{\alpha}(\X, \Y) \leq 2 \big( \text{$\AGW$}_{\alpha}(\X, \Z) + \text{$\AGW$}_{\alpha}(\Z, \Y) \big)$.
\end{enumerate}
\end{proposition}

These basic properties ensure that our new divergence is well-posed. However, the most intriguing question is what invariances a convex combination of GW and COOT exhibit? Intuitively, we expect that AGW inherits the common invariants of both. Formally, we first introduce a necessary notion of weak invariance below. 
\begin{definition}
    We call $D = \inf_{\pi \in \Pi} F(\pi, \X, \Y)$, where $\X, \Y$ are input data and $\Pi$ is a set of feasible couplings, an OT-based divergence. Then $D$ is weakly invariant to translation if for every $a, b \in \mathbb R$, we have $\inf_{\pi \in \Pi} F(\pi, \X, \Y) = C + \inf_{\pi \in \Pi} F(\pi, \X+a, \Y+b)$, for some constant $C$ depending on $a, b, \X, \Y$ and $\Pi$. 
\end{definition}
Here, we denote the translation of $\X$ as $\X + a$, whose elements are of the form $\X_{i, j} + a$. In other words, an OT-based divergence is weakly invariant to translation if only the optimal transport plan is preserved under translation, but not necessarily the divergence itself. We now state our result regarding the weak invariance of AGW below.
\begin{theorem}[Invariant property]
\label{prop:invariant}
    \text{ }
    \begin{enumerate}
        \item $\AGW$ is weakly invariant to translation.

        \item Given two input matrices $\X$ and $\Y$, if $\mu = \nu$ and $\Y$ is obtained by permuting columns (features) of $\X$ via the permutation $\sigma_c$ (so $\nu' = (\sigma_c)_{\#} \mu'$), then $\text{$\AGW$}(\X, \Y) = 0$.
    \end{enumerate}
\end{theorem}
It is well known that, in Euclidean space, there are only three types of isometry: translation, rotation, and reflection (see \cite{Konrad}, for example). $\AGW$ inherits the weak invariant to translation from COOT. In practice, we would argue that the ability to preserve the optimal plan under translation is much more important than preserving the distance itself. In other words, the translation only shifts the minimum but has no impact on the optimization procedure, meaning that the minimizer remains unchanged. Similar to GW distance, AGW also covers basic isometries, for example feature swap. Logically, AGW covers much fewer isometries than GW distance, since AGW only has at most as finitely many isometries as COOT, whereas GW distance has infinitely many isometries. Given the superior performance of AGW over GW and COOT in the experiments, we conjecture that there may be other relevant isometries. We leave a more detailed understanding of the isometries induced by AGW to future work.

\subsection{Related work}

In optimal transport, a common approach to incorporate prior knowledge is integrating it into the cost function. For example, when all classification labels are available, \cite{Melis20} proposed the Optimal Transport Dataset Distance to compare two datasets by adding the Wasserstein distance between the histograms associated with the labels, to the transport cost. However, this approach is only applicable when the data lives in the same ground spaces. Given some known matching between samples across domains, \cite{Gu2022} used the Keypoint-Guided Optimal Transport to capture more efficiently the discrepancy between two datasets, by attaching a mask matrix containing alignment information to the transport plan via the element-wise multiplication. Another line of work is the Fused Gromov-Wasserstein (FGW) distance \cite{vay2019fgw} used for comparing structured objects at both structural and feature levels. The objective function of FGW is a convex combination between the GW term defined based on the intra-domain structural information, and the Wasserstein term that takes into account the information about the features associated to the corresponding structural elements.

Despite the resemblance to FGW, $\AGW$ serves a very different purpose and covers different use cases. First and foremost, AGW is a divergence that tackles cross-domain applications that are inaccessible to FGW. Second, FGW is mostly used for structured objects endowed with additional feature information, while AGW can be used on empirical measures defined for any set of objects. Finally, the feature space in the case of FGW is associated with the sample space, whereas in $\AGW$ the two spaces are independent. We would also like to stress that the notion of feature space in FGW \textit{does not} have the same meaning as the one in $\AGW$ (and COOT). Each element of the former is associated with a point in the sample space; for example, each node of the graph may be colored by a specific color (feature). By contrast, the feature information in AGW is precisely the coordinates of a point, in addition to its representation in the original and dissimilarity-induced spaces.

\renewcommand{\arraystretch}{1.15}
\section{Experimental evaluations}
In this section, we present the empirical evaluations of the proposed divergence for the single-cell multi-omics alignment task and a heterogeneous domain adaptation task in  ML. Overall, our experiments answer the following questions:
\begin{enumerate}
    \item[Q1.] Does tightening the invariances of GW improve the performance in downstream tasks where it was previously used?
    \item[Q2.] Does prior knowledge introduced in GW help in obtaining better cross-domain matchings?
\end{enumerate}
We particularly focus on the emerging OT-driven single-cell alignment task for two reasons. First, GW imposed itself as a state-of-the-art method for this task \cite{Pamona, Demetci2022,UniPort}, and thus it is important to see whether we can improve upon it using AGW. Second, several single-cell benchmark datasets provide ground-truth matchings on the feature level in addition to the common sample alignments. This information allows us to assess the importance of guiding cross-domain matching with partial or full knowledge of the relationships between the features in the two domains. 

In the following experiments, we also consider the entropic regularization on both sample and feature couplings optimized by AGW. For all experiments, we detail our experimental setup, including the hyperparameter tuning procedure, as well as the runtime of the algorithms, in the Appendix. 

\subsection{Integrating single-cell multi-omics datasets}
Integration of data from different single-cell sequencing experiments is an important task in biology for which OT has proven to be useful\cite{Pamona, UniPort, Demetci2020}. Single-cell experiments measure various genomic features and events at the individual cell resolution. Jointly studying these can give scientists insight into the mechanisms regulating cells \cite{Pamona, Demetci2020, bindSC}. However, experimentally combining multiple types of measurements for the same cell is challenging for most combinations. To study the relationships and interactions between different aspects of the genome, scientists rely on computational integration of multi-modal data taken on different but related cells (e.g., by cell type or tissue). 

\paragraph{Single-cell alignment} Below, we follow \cite{Demetci2020} and align samples (i.e., cells) of simulated and real-world single-cell datasets from different measurement modalities in order to perform the integration. For all datasets, we have ground-truth information on cell-cell alignments, which we only use for benchmarking. We demonstrate in Table \ref{table:scCells} that our proposed framework yields higher quality cell alignments (with lower alignment error) compared to both GW and COOT.

\begin{table}[H]
\caption{\label{table:scCells}\textbf{Single-cell alignment error},  as quantified by the average `fraction of samples closer than true match' (FOSCTTM) metric (lower values are better, metric defined in the Appendix). Cell alignment performance of AGW is compared against the alignments by GW (SCOT), COOT, and bindSC, which performs bi-order canonical correlation analysis for alignment (detailed in the next section). It requires prior information on feature relationships, which we do not have for the first three simulations (thus the N/A).}
\begin{center}
\resizebox{.95\linewidth}{!}{
\begin{tabular}{@{}l|ccccccc@{}}

\textbf{}          & \textbf{Sim 1}           & \textbf{Sim 2}  & \textbf{Sim 3}  & \textbf{\begin{tabular}[c]{@{}c@{}}Splatter Simulation\\ (Synthetic RNA-seq)\end{tabular}} & \textbf{scGEM} & \textbf{SNARE-seq} & \textbf{CITE-seq} \\ \hline
\textbf{AGW}      & \textbf{0.073}           & \textbf{0.0041} & \textbf{0.0082} & \textbf{0.0}                                                                               & \textbf{0.183} & \textbf{0.136}     & \textbf{0.091}    \\
\textbf{GW (SCOT)} & 0.0866                  & 0.0216          & \underline{0.0084}          & \underline{7.1 e-5}                                                                                    & \underline{0.198}          & \underline{0.150}              & \underline{0.131}             \\
\textbf{COOT}      & \underline{0.0752} & \textbf{0.0041} & 0.0088          & \textbf{0.0}                                                                               & 0.206          & 0.153              & 0.132             \\
\textbf{bindSC}    & N/A                      & N/A             & N/A             & 3.8 e-4                                                                                    & 0.204          & 0.242              & 0.144             \\
\end{tabular}}
\end{center}
\end{table}

\paragraph{Alignment of genomic features} AGW augements GW formulation with a feature coupling matrix. Therefore, we also jointly align features and investigate AGW's use for revealing potential biological relationships. All current single-cell alignment methods can only align samples (i.e., cells). A nuanced exception is bindSC \cite{bindSC}, which performs bi-order canonical correlation analysis to integrate cells. As a result, it internally generates a feature correlation matrix that users can extract. Among all the real-world datasets in Table \ref{table:scCells}, CITE-seq \cite{CITEseq} is the only one with ground-truth information on feature correspondences. This dataset has paired single-cell measurements on the abundance levels of 25 antibodies, as well as activity (i.e., ``expression'') levels of genes, including the genes that encode these 25 antibodies. So, we first present unsupervised feature alignment results on the CITE-seq dataset. For completion, we also report the biological relevance of our feature alignments on SNARE-seq \cite{SNAREseq} and scGEM \cite{scGEM} datasets in Appendix Section 4. However, note that these datasets (unlike CITE-seq) do not have clear ground-truth feature correspondences.  We compare our feature alignments with bindSC and COOT in Figure \ref{fig:cite}. The entries in the feature alignment matrices are arranged such that the ``ground-truth'' correspondences lie in the diagonal, marked by green squares. While AGW correctly assigns 19 out of 25 antibodies to their encoding genes with the highest alignment probability, this number is 15 for COOT and 13 for bindSC (which yields correlation coefficients instead of alignment probabilities). Additionally, the OT methods yield more sparse alignments thanks to the ``least effort'' requirement in their formulation.
\begin{figure}[!t]
\centering
\includegraphics[width=1.0\linewidth]{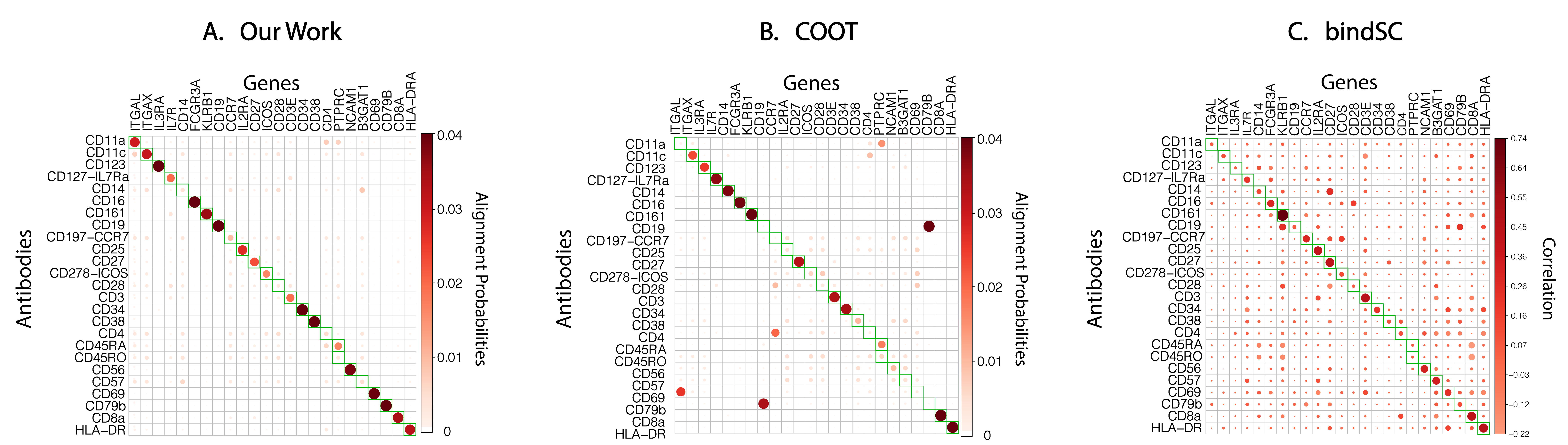}
\caption{\label{fig:cite} Feature alignment matrices in the CITE-seq dataset. AGW provides more qualitative results than previously proposed methods.}
\end{figure}
\paragraph{Importance of prior knowledge}
Finally, we demonstrate the advantage of providing prior information by aligning a multi-species gene expression dataset, which contains measurements from the adult mouse prefrontal cortex\cite{mouse} and pallium of bearded lizard \cite{lizard}. Since measurements come from two different species, the feature space (i.e. genes) differs, and there is also no 1-1 correspondence between the samples (i.e. cells). However, there is a shared subset within the features, i.e., paralogous genes, which are genes that descend from a common ancestor of the two species and have similar biological functions. We also have some domain knowledge about the cells that belong to similar cell types across the two species. Thus, we expect AGW to recover these relationships in both the sample and the feature alignment matrices. 

\setlength{\columnsep}{10pt}%
\setlength{\intextsep}{0pt}
\begin{wrapfigure}[15]{r}{.42\linewidth}
\centering
\includegraphics[width=\linewidth]{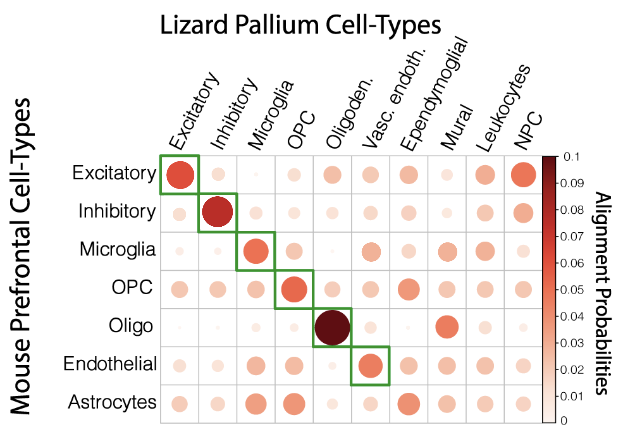}
\caption{\textbf{Cell-type alignment results on cross-species dataset.}}
\label{fig:crossAlignments}
\end{wrapfigure}

Figure \ref{fig:crossAlignments} visualizes the cell-type alignment probabilities yielded by AGW when full supervision is provided on the $10,816$ paralogous genes. The green boxes indicate alignment between similar types of cells. This matrix is obtained by averaging the sample alignment matrix (i.e., cell-cell alignments) into cell-type groups. Figure \ref{fig:crossAlignments} demonstrates that AGW yielded biologically plausible alignments, as all the six cell types that have a natural match across the two species are correctly matched. We additionally show in Tables \ref{table:xSp-supervisedCell} and \ref{table:xSp-supervisedGene} that providing supervision on one level of alignment (e.g., features) improves the alignment quality on the other level (e.g., samples). Supervision scheme is detailed in the ``Experimental Set-up'' section of the Appendix.

\begin{table}[H]
\caption{\label{table:xSp-supervisedCell} AGW's sample (i.e. cell) alignment performance with increasing supervision on feature (i.e. gene) alignments}
\begin{center}
\resizebox{.95\linewidth}{!}{
\begin{tabular}{@{}l|cccccc|l@{}}
\toprule
\textbf{Supervision on orthologous genes}    & 0\%   & 20\%  & 40\%  & 60\% & 80\% & 100\% & \textbf{bindSC} \\ \midrule
\textbf{\% Accuracy of Cell-type alignment} & 66.67 & 83.34 & 83.34 & 100  & 100  & 100   & 66.67           \\ \bottomrule
\end{tabular}
}
\end{center}
\end{table}

\begin{table}[H]
\caption{\label{table:xSp-supervisedGene}AGW's feature (i.e. paralogous gene) alignment performance with increasing supervision on sample (i.e. cell) alignments. Note that bindSC takes in paralogous gene pair information at initialization (as a feature alignment prior), so we are unable to include it as a valid baseline here.}
\begin{center}
\resizebox{.95\linewidth}{!}{
\begin{tabular}{@{}l|cccccc|c@{}}
\toprule
\textbf{Supervision on cell-type alignment} & 0\%   & 20\%  & 40\%  & 60\%  & 80\%  & 100\% & \textbf{bindSC} \\ \midrule
\textbf{\% Accuracy of gene alignment}      & 73.53 & 82.92 & 87.55 & 91.32 & 93.82 & 93.82 & N/A             \\ \bottomrule
\end{tabular}
}
\end{center}
\end{table}

\subsection{Heterogeneous domain adaptation}
We now demonstrate the generalizability of our approach and turn our attention to an ML task, heterogeneous domain adaptation, where COOT and GW were previously successfully used. Domain adaptation (DA) refers to the problem in which a classifier learned on one domain (called \textit{source}) can generalise well to the other one (called \textit{target}). Here, we illustrate an application of AGW in unsupervised and semi-supervised heterogeneous DA (HDA), where the samples in source and target domains live in different spaces, and we only have access to as few as zero labeled target samples.

\paragraph{Setup} We follow the evaluation setup from \cite{COOT}: AGW, GW, and COOT are evaluated on source-target pairs from the Caltech-Office dataset \cite{Saenko10}. We consider all pairs between the three domains: 
Amazon (A), Caltech-$256$ (C), and Webcam (W), whose images are embeddings extracted from the second last layer in the GoogleNet \cite{Szegedy15} (vectors in $\mathbb R^{4096}$) and CaffeNet \cite{Jia14} (vectors in $\mathbb R^{1024}$) neural network architectures. In the semi-supervised setting, we incorporate the prior knowledge of the target labels by adding an additional cost matrix to the training of sample coupling, so that a source sample will be penalized if it transfers mass to the target samples in the different classes. Once the sample coupling $\GGs$ is learned, we obtain the final prediction using label propagation:  
$\widehat{y}_t = \arg\max_k L_{k\cdot}$
where $L = D_s \GGs$ and $D_s$ denotes one-hot encodings of the source labels $y_s$. The interpolation and entropic regularization hyperparameters are tuned, using accuracy as the evaluation metric.

\paragraph{Results} Table 4 demonstrates each method's performance averaged across ten runs, after hyperparameter tuning with the same grid of values for equivalent hyperparameters (e.g. entropic regularization, details in Appendix). We consider an unsupervised case, and a semi-supervised case with 3 samples used for supervision. Table 4 shows that AGW tends to outperform both GW and COOT, which supports our claim about its capacity to properly adjust the invariance to the datasets at hand. 

\begin{table}[h]
\caption{\label{tab:hda} \textbf{Heterogeneous domain adaptation results}. Best results are bolded. In the ``AGW (Best $\alpha$)'' column, the $\alpha$ values used are: $0.6, 0.9, 0.7, 0.9, 0.3, 0.8, 0.7, 0.2, 0.6$ top to bottom for the unsupervised setting, and $0.2, 0.1, 0.2, 0.7, 0.2, 0.9, 0.8, 0.9, 0.4$ for the semi-supervised setting.}
\begin{center}
\resizebox{.98\linewidth}{!}{
\begin{tabular}{@{}lcccc|cccc@{}}
\toprule
                           & \multicolumn{4}{c|}{\textbf{Unsupervised}}                                                                                                                                          & \multicolumn{4}{c}{\textbf{Semi-supervised (t=3)}}                                                                                                                                                     \\ \midrule
                           & \textbf{COOT}   & \textbf{GW}     & \textbf{\begin{tabular}[c]{@{}c@{}}AGW \\ ($\alpha$=0.5)\end{tabular}} & \textbf{\begin{tabular}[c]{@{}c@{}}AGW\\ (Best $\alpha$)\end{tabular}} & \textbf{COOT}  & \textbf{GW}                   & \textbf{\begin{tabular}[c]{@{}c@{}}AGW \\ ($\alpha$=0.5)\end{tabular}} & \textbf{\begin{tabular}[c]{@{}c@{}}AGW\\ (Best $\alpha$)\end{tabular}} \\ \midrule
\textbf{A $\rightarrow$ A} & 50.3 $\pm$ 15.9 & 86.2 $\pm$ 2.3  & 90.5 $\pm$ 2.4                                                   & \textbf{93.1 $\pm$ 1.6}                                 & 91.1 $\pm$ 2.0 & 93.2 $\pm$ 0.9                & 93.8 $\pm$ 1.3             & \textbf{96.0 $\pm$ 0.8}                                \\
\textbf{A $\rightarrow$ C} & 35.0 $\pm$ 6.4  & 64.1 $\pm$ 6.2  & 68.2 $\pm$ 7.4                                                   & \textbf{68.3 $\pm$ 14.1}                                & 59.7 $\pm$ 3.6 & 92.8 $\pm$ 2.1                & 90.7 $\pm$ 1.9                                                         & \textbf{93.5 $\pm$ 1.8}                                 \\
\textbf{A $\rightarrow$ W} & 39.8 $\pm$ 14.5 & 79.6 $\pm$ 11.1 & 75.5 $\pm$ 3.1                                                         & \textbf{79.8 $\pm$ 3.5}                                 & 72.6 $\pm$ 4.4 & 91.6 $\pm$ 1.8                & 91.4 $\pm$ 1.1                                                         & \textbf{93.8 $\pm$ 0.7}                                 \\
\textbf{C $\rightarrow$ A} & 40.8 $\pm$ 15.8 & 53.0 $\pm$ 13.2 & 48.5 $\pm$ 6.9                                                         & \textbf{55.4 $\pm$ 7.1}                                 & 83.1 $\pm$ 5.1 & 81.2 $\pm$ 1.2                & { 84.3 $\pm$ 1.6}                                                   & \textbf{85.6 $\pm$ 1.2}                                 \\
\textbf{C $\rightarrow$ C} & 33.4 $\pm$ 10.7 & \textbf{81.9 $\pm$ 30.5} & 68.5 $\pm$ 5.5                                                         & 76.4 $\pm$ 5.6                                          & 59.3 $\pm$ 8.4 & 85.3 $\pm$ 2.8                & 83.4 $\pm$ 2.3                                                         & \textbf{86.5 $\pm$ 2.1}                                 \\
\textbf{C $\rightarrow$ W} & 37.5 $\pm$ 10.4 & 53.5 $\pm$ 15.9 & { 56.6 $\pm$ 7.6}                                                   & \textbf{57.7 $\pm$ 14.3}                                & 64.6 $\pm$ 6.2 & 79.7 $\pm$ 2.5                & 81.3 $\pm$ 4.3                                                         & \textbf{83.2 $\pm$ 2.4}                                 \\
\textbf{W $\rightarrow$ A} & 44.3 $\pm$ 14.0 & 50.4 $\pm$ 22.1 & { 52.1 $\pm$ 3.8}                                                   & \textbf{60.1 $\pm$ 9.1}                                 & 94.3 $\pm$ 2.2 & 93.4 $\pm$ 5.2                & 92.3 $\pm$ 1.5                                                         & \textbf{97.1 $\pm$ 0.8}                                 \\
\textbf{W $\rightarrow$ C} & 27.4 $\pm$ 10.2 & 54.3 $\pm$ 14.7 & 53.6 $\pm$ 17.3                                                        & \textbf{60.9 $\pm$ 13.3}                                & 55.0 $\pm$ 7.1 & 90.9 $\pm$ 3.5                & 90.9 $\pm$ 2.0                                                         & \textbf{94.7 $\pm$ 1.1}                                 \\
\textbf{W $\rightarrow$ W} & 57.9 $\pm$ 13.4 & 92.5 $\pm$ 2.6  & 90.3 $\pm$ 5.4                                                         & \textbf{97.2 $\pm$ 0.9}                                 & 87.4 $\pm$ 4.4 & 97.4 $\pm$ 2.6\ & { 98.5 $\pm$ 0.7}                                                   & \textbf{98.7 $\pm$ 0.5}                                 \\
\hline
\end{tabular}
}

\end{center}
\end{table}

\section{Discussion and conclusion}

We present a new OT-based distance for incomparable spaces called augmented Gromov-Wasserstein (AGW), which relies on the GW distance and CO-Optimal transport. This novel metric allows us to narrow down the choices of isometries induced by GW distance, while better exploiting the prior knowledge or the input data. We study its basic properties and empirically show that such restriction results in better performance for single-cell multi-omic alignment tasks and transfer learning. Future work will focus on refining the theoretical analysis of the isometries induced by AGW distance, which may shed light on understanding why they are useful and relevant in various learning tasks. It would also be interesting to extend this framework to the unbalanced, and/or continuous setting and other tasks where feature supervision proposed by domain experts may be incorporated in OT framework. 

\paragraph{Limitations} One limitation of AGW is the inherent computational burden of the GW component. Possible solutions can be considering low-rank coupling and cost matrix \cite{Scetbon22}, or using the divide and conquer strategy \cite{Chowdhury21}, which allows us to scale the GW distance up to a million points. 
\newpage
\section{Appendix}\label{appendix}
\subsection{Proofs}

\begin{proof}[Proof of \cref{prop:basic_prop}]
The proof of this proposition can be adapted directly from \cite{vay2019fgw}. For self-contained purpose, we give the proof here. Denote
\begin{itemize}
    \item $(\GGs_{\alpha}, \GGv_{\alpha})$ the optimal sample and feature couplings for $\text{AGW}_{\alpha}(\X, \Y)$.

    \item $(\GGs_0, \GGv_0)$ the optimal sample and feature couplings for $\text{COOT}(\X, \Y)$.

    \item $\GGs_1$ the optimal sample coupling for $\text{GW}(\X, \Y)$.
\end{itemize}
Due to the suboptimality of $\GGs_{\alpha}$ for GW and $(\GGs_1, \GGv_0)$ for AGW, we have
\begin{align*}
    \alpha \langle L(\D_X, \D_Y) \otimes \GGs_1, \GGv_1 \rangle &\leq \alpha \langle L(\D_X, \D_Y) \otimes \GGs_{\alpha}, \GGs_{\alpha} \rangle + (1 - \alpha) \langle L(\X, \Y) \otimes \GGv_{\alpha}, \GGs_{\alpha} \rangle \\
    &\leq \alpha \langle L(\D_X, \D_Y) \otimes \GGs_1, \GGs_1 \rangle + (1 - \alpha) \langle L(\X, \Y) \otimes \GGv_0, \GGs_1 \rangle,
\end{align*}
or equivalently
\begin{align}
    \alpha \text{GW}(\X, \Y) \leq \text{AGW}_{\alpha}(\X, \Y) \leq \alpha \text{GW}(\X, \Y) + (1 - \alpha) \langle L(\X, \Y) \otimes \GGv_0, \GGs_1 \rangle.
\end{align}
Similarly, we have
\begin{align}
    (1 - \alpha) \text{COOT}(\X, \Y) \leq \text{AGW}_{\alpha}(\X, \Y) \leq (1 - \alpha) \text{COOT}(\X, \Y) + \alpha \langle L(\D_X, \D_Y) \otimes \GGs_0, \GGs_0 \rangle.
\end{align}
The interpolation property then follows by the sandwich theorem.

Regarding the relaxed triangle inequality, given three triples $(\X, \mu_{sx}, \mu_{fx}), (\Y, \mu_{sy}, \mu_{fy})$ and $(\Z, \mu_{sz}, \mu_{fz})$, let $(\pi^{XY}, \gamma^{XY}), (\pi^{YZ}, \gamma^{YZ})$ and $(\pi^{XZ}, \gamma^{XZ})$ be solutions of the problems $\text{AGW}_{\alpha}(\X, \Y), \text{AGW}_{\alpha}(\Y, \Z)$ and $\text{AGW}_{\alpha}(\X, \Z)$, respectively. Denote $P = \pi^{XY} \text{diag}\left( \frac{1}{\mu_{sy}} \right) \pi^{YZ}$ and $Q = \gamma^{XY} \text{diag}\left( \frac{1}{\mu_{fy}} \right) \gamma^{YZ}$. Then, it is not difficult to see that $P \in \Pi(\mu_{sx}, \mu_{sz})$ and $Q \in \Pi(\mu_{fx}, \mu_{fz})$. The suboptimality of $(P,Q)$ implies that 
\begin{align*}
    &\frac{\text{AGW}_{\alpha}(\X, \Z)}{2} \\
    &\leq \alpha \sum_{i,j,k,l} \frac{|\D_X(i,j) - \D_Z(k,l)|^2}{2} P_{i,k} P_{j, l} + (1 - \alpha) \sum_{i,j,k,l} \frac{|\X_{i,j} - \Z_{k,l}|^2}{2} P_{i,k} Q_{j,l} \\
    &= \alpha \sum_{i,j,k,l} \frac{|\D_X(i,j) - \D_Z(k,l)|^2}{2} \left( \sum_e \frac{\pi^{XY}_{i,e} \pi^{YZ}_{e,k}}{(\mu_{sy})_e} \right) \left(\sum_o \frac{\pi^{XY}_{j,o} \pi^{YZ}_{o,l}}{(\mu_{sy})_o} \right) \\
    &+ (1 - \alpha) \sum_{i,j,k,l} \frac{|\X_{i,j} - \Z_{k,l}|^2}{2} \left(\sum_e \frac{\pi^{XY}_{i,e} \pi^{YZ}_{e,k}}{(\mu_{sy})_e} \right) \left( \sum_o \frac{\gamma^{XY}_{j,o} \gamma^{YZ}_{o,l}}{(\mu_{fy})_o} \right) \\
    &\leq \alpha \sum_{i,j,k,l,e,o} |\D_X(i,j) - \D_Y(e,o)|^2 \frac{\pi^{XY}_{i,e} \pi^{YZ}_{e,k}}{(\mu_{sy})_e} \frac{\pi^{XY}_{j,o} \pi^{YZ}_{o,l}}{(\mu_{sy})_o} + (1 - \alpha) \sum_{i,j,k,l,e,o} |\X_{i,j} - \Y_{e,o}|^2 \frac{\pi^{XY}_{i,e} \pi^{YZ}_{e,k}}{(\mu_{sy})_e} \frac{\gamma^{XY}_{j,o} \gamma^{YZ}_{o,l}}{(\mu_{fy})_o} \\
    &+ \alpha \sum_{i,j,k,l,e,o} |\D_Y(e, o) - \D_Z(k,l)|^2 \frac{\pi^{XY}_{i,e} \pi^{YZ}_{e,k}}{(\mu_{sy})_e} \frac{\pi^{XY}_{j,o} \pi^{YZ}_{o,l}}{(\mu_{sy})_o} + (1 - \alpha) \sum_{i,j,k,l,e,o} |\Y_{e, o} - \Z_{k,l}|^2 \frac{\pi^{XY}_{i,e} \pi^{YZ}_{e,k}}{(\mu_{sy})_e} \frac{\gamma^{XY}_{j,o} \gamma^{YZ}_{o,l}}{(\mu_{fy})_o} \\
    &= \alpha \sum_{i,j,e,o} |\D_X(i,j) - \D_Y(e,o)|^2 \pi^{XY}_{i,e} \pi^{XY}_{j,o} + (1 - \alpha) \sum_{i,j,e,o} |\X_{i,j} - \Y_{e,o}|^2 \pi^{XY}_{i,e} \gamma^{XY}_{j,o} \\
    &+ \alpha \sum_{k, l, e, o} |\D_Y(e, o) - \D_Z(k,l)|^2 \pi^{YZ}_{e,k} \pi^{YZ}_{o,l} + (1 - \alpha) \sum_{k,l,e,o} |\Y_{e, o} - \Z_{k,l}|^2 \pi^{YZ}_{e,k} \gamma^{YZ}_{o,l} \\
    &= \text{AGW}_{\alpha}(\X, \Y) + \text{AGW}_{\alpha}(\Y, \Z).
\end{align*}
where the second inequality follows from the inequality: $(x + y)^2 \leq 2(x^2 + y^2)$.
\end{proof}

\begin{corollary}
\label{prop:coot_invariant}
    COOT is weakly invariant to translation.
\end{corollary}
\begin{proof}[Proof of \cref{prop:coot_invariant}]
It is enough to show that, for any $c \in \mathbb R$, we have $\text{COOT}(\X, \Y + c) = \text{COOT}(\X, \Y) + C$, for some constant $C$. 
Indeed, given $\GGs \in \Pi(\mu, \nu), \GGv \in \Pi(\mu', \nu')$, for any $c \in \mathbb R$,  
\begin{align}
    \sum_{ijkl} (\X_{ik} - \Y_{jl} - c)^2 \GGs_{ij} \GGv_{kl} &= \sum_{ijkl} (\X_{ik} - \Y_{jl})^2 \GGs_{ij} \GGv_{kl} - 2c \sum_{ijkl} (\X_{ik} - \Y_{jl}) \GGs_{ij} \GGv_{kl} + c^2
\end{align}
Now,
\begin{align}
    \sum_{ijkl} (\X_{ik} - \Y_{jl}) \GGs_{ij} \GGv_{kl} &= \sum_{ijkl} \X_{ik} \GGs_{ij} \GGv_{kl} - \sum_{ijkl} \Y_{jl} \GGs_{ij} \GGv_{kl} \\
    &= \sum_{ik} \X_{ik} \left( \sum_j \GGs_{ij} \right) \left( \sum_l \GGv_{kl} \right) - \sum_{jl} \Y_{jl} \left( \sum_i \GGs_{ij} \right) \left( \sum_k \GGv_{kl} \right) \\
    &= \sum_{ik} \X_{ik} \mu_i \mu'_k - \sum_{jl} \Y_{jl} \nu_j \nu'_l \\
    &= \mu ^T \X \mu' - \nu ^T \Y \nu'.
\end{align}
So,
\begin{align*}
    \text{COOT}(\X, \Y + c) = \text{COOT}(\X, \Y) - 2 c \left( \mu^T \X \mu' - \nu^T \Y \nu' \right) + c^2.
\end{align*}
This implies that COOT is weakly invariant to translation.  
\end{proof}

\begin{proof}[Proof of \cref{prop:invariant}]
Note that the GW term in AGW remains unchanged by translation. By adapting the proof of \cref{prop:coot_invariant}, we obtain
\begin{align*}
    \text{AGW}_{\alpha}(\X, \Y + c) = \text{AGW}_{\alpha}(\X, \Y) - 2 c \left( \mu^T \X \mu' - \nu^T \Y \nu' \right) + c^2.
\end{align*}
This means AGW is weakly invariant to translation.

Note that $\Y = \X Q$, where $Q$ is a permutation matrix corresponding to the permutation $\sigma_c$. Since $\Y$ is obtained by swapping columns of $\X$, we must have that $\text{GW}(\X, \Y) = 0$ and the optimal plan between $\X$ and $\Y$ is $\GGs = \frac{1}{n^2} \text{Id}_n$. Similarly, $\text{COOT}(\X, \Y) = 0$ and $\GGs, \GGv = \frac{1}{n}Q$ are the optimal sample, feature couplings, respectively. In other words, $\langle L(\D_{X},\D_{Y})\otimes \GGs, \GGs \rangle = 0$ and $\langle L(\X,\Y) \otimes \GGv, \GGs \rangle = 0$. We deduce that $\text{AGW}_{\alpha}(\X, \Y) = 0$.
\end{proof}

\subsection{Additional illustrations on the MNIST and USPS handwritten digits}
\begin{figure}[H]
    \centering
    \includegraphics[width=\linewidth]{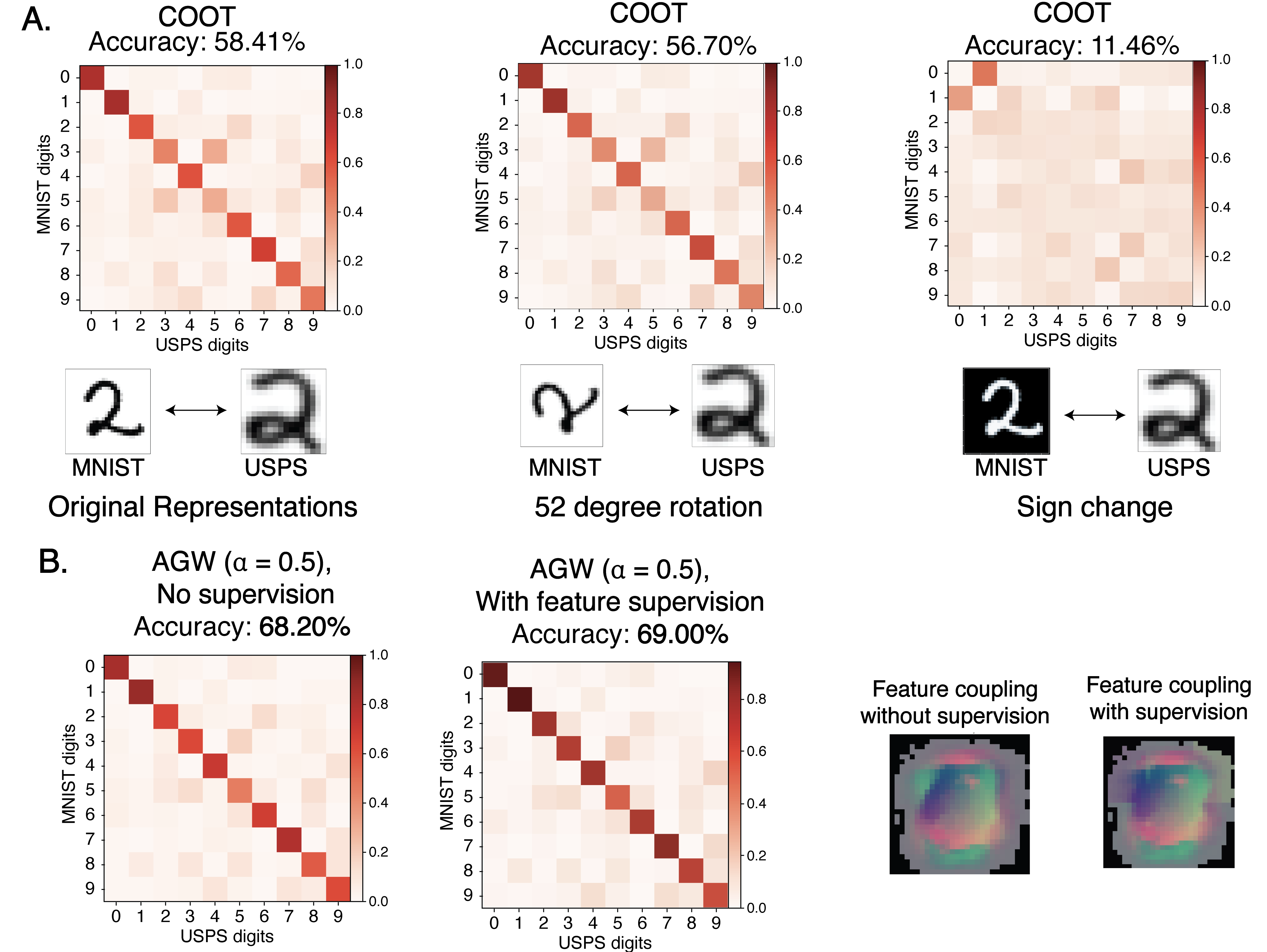}
    \caption{\textbf{A.} Examples of isometric transformations that COOT is not invariant to (e.g. rotation, sign change), \textbf{B.} Further improving MNIST-USPS alignment accuracy of AGW through supervision on features to restrict reflections along the y-axis.\label{fig:SI-MNIST}}
\end{figure}

\section{Experimental Set-Up Details}
\subsection{Code availability and access to scripts for experiments } Code and datasets used in this paper can be found at: \url{https://github.com/pinardemetci/AGW}
\subsection{MNIST Illustrations} We align $1000$ images of hand-written digits from the MNIST dataset with $1000$ images from the USPS dataset. Each dataset is subsampled to contain $100$ instances of each of the $10$ possible digits ($0$ through $9$), using the random seed of $1976$. We set all marginal distributions to uniform, and use cosine distances for GW and AGW. For all methods, we consider both the entropically regularized and also the non-regularized versions. For entropic regularization, we sweep a grid of $\epsilon_1, \epsilon_2 (\textrm{if applicable}) \in [5e-4, 1e-3, 5e-3, 1e-2, 5e-2, 1e-1, 5e-1] $. For AGW, we consider $[0.1, 0.2, 0.3, ..., 0.9]$, and present results with the best performing hyperparameter combination of each method, as measured by the percent accuracy of matching images from the same digit across the two datasets.

\subsection{Single-cell multi-omic alignment experiments}
As a real-world application of AGW, we align single-cell data from different measurement domains. Optimal transport has recently been applied to this problem in computational biology by multiple groups \cite{Demetci2020, Pamona, Demetci2022, UniPort}. To briefly introduce the problem: Biologists are interested in jointly studying multiple genomic (i.e. ``multi-omic'') aspects of cells to determine biologically-relevant patterns in their co-variation. Such studies could reveal how the different molecular aspects of a cell's genome (e.g. its 3D structure, reversible chemical modifications it undergoes, activity levels of its genes etc) interact to regulate the cell's response to its environment. These studies are of interest for both fundamental biology research, as well as drug discovery applications. However, as Liu \textit{et al} describes \cite{liu_et_al:LIPIcs:2019:11040}, it is experimentally difficult to combine multiple measurements on the same cells. Consequently, computational approaches are developed to integrate data obtained from different measurement modalities using biologically relevant cell populations. In this paper, we apply AGW to jointly align both cells and genomic features of single-cell datasets. This is a novel direction in the application of optimal transport (OT) to single-cell multi-omic alignment task, as the existing OT-based algorithms only align cells. 

\paragraph{Datasets} We largely follow the first paper that applied OT to single-cell multi-omic alignment task \cite{Demetci2020} in our experimental set-up and use four simulated datasets and three real-world single-cell multi-omic datasets to benchmark our cell alignment performance. Three of the simulated datasets have been generated by Liu \textit{et al.} \cite{liu_et_al:LIPIcs:2019:11040} by non-linearly projecting 600 samples from a common 2-dimensional space onto different 1000- and 2000- dimensional spaces with 300 samples in each. In the first simulation, the data points in each domain form a bifurcating tree structure that is commonly seen in cell populations undergoing differentiation. The second simulation forms a three dimensional Swiss roll. Lastly, the third simulation forms a circular frustum that resembles what is commonly observed when investigating cell cycle. These datasets have been previously used for benchmarking by other cell-cell alignment methods \cite{liu_et_al:LIPIcs:2019:11040,singh20,cao2020unsupervised, Pamona,Demetci2020}. We refer to these datasets as ``Sim 1'', ``Sim 2'', and ``Sim 3'', respectively. We include a fourth simulated dataset that has been generated by \cite{Demetci2020} using a single-cell RNA-seq data simulation package in R, called Splatter \cite{zappia2017splatter}. We refer to this dataset as ``Synthetic RNA-seq''. This dataset includes a simulated gene expression domain with 50 genes and 5000 cells divided across three cell-types, and another domain created by non-linearly projecting these cells onto a 500-dimensional space. As a result of their generation schemes, all simulated datasets have ground-truth 1-1 cell correspondence information. We use this information solely for benchmarking. We do not have access to ground-truth feature relationships in these datasets, so, we exclude them from feature alignment experiments.

Additionally to the simulated datasets, we include three real-world sequencing datasets in our experiments. To have ground-truth information on cell correspondences for evaluation, we choose three co-assay datasets which have paired measurements on the same individual cells: an scGEM dataset \cite{cheow2016}, a SNARE-seq dataset \cite{SNAREseq}, and a CITE-seq dataset \cite{CITEseq} (these are exceptions to the experimental challenge described above). These first two datasets have been used by existing OT-based single-cell alignment methods \cite{cao2020unsupervised, singh20, Demetci2020, Pamona, Demetci2022}, while the last one was included in the evaluations of a non-OT-based alignment method, bindSC \cite{bindSC} (described in the ``Evaluations'' section below). The scGEM dataset contains measurements on gene expression and DNA methylation states of 177 individual cells from human somatic cell population undergoing conversion to induced pluripotent stem cells (iPSCs) \cite{cheow2016}. We accessed the pre-processed count matrices for this dataset through the following GitHub repository: \url{https://github.com/caokai1073/UnionCom}. The SNARE-seq dataset contains gene expression and chromatin accessibility profiles of 1047 individual cells from a mixed population of four cell lines: H1(human embryonic stem cells), BJ (a fibroblast cell line), K562 (a lymphoblast cell line), and GM12878 (lymphoblastoid cells derived from blood) \cite{SNAREseq}. We access their count matrices on Gene Expression Omnibus platform online, with the accession code GSE126074. Finally, the CITE-seq dataset has gene expression profiles and epitope abundance measurements on 25 antibodies from 30,672 cells from human bone marrow tissue \cite{CITEseq}. The count matrices for this dataset were downloaded from the Seurat website \footnote{\url{https://satij alab.org/seurat/v4.0/weighted_nearest_neighbor_analysis.html}}. We use these three real-world single-cell datasets for both cell-cell (i.e. sample-sample) alignment benchmarking, as well as feature-feature alignment benchmarking. In addition to these three datasets, we include a fourth single-cell datasets, which contains data from the same measurement modality (i.e. gene expression), but from two different species: mouse \cite{mouse} and bearded lizard \cite{lizard}. Our motivation behind including this dataset is to demonstrate the effects of both sample-level (i.e. cell-level) and feature-level supervision on alignment qualities. We refer to this dataset as the ``cross-species dataset'', which contains 4,187 cells from lizard pallium (a brain region) and 6,296 cells from the mouse prefrontal cortex. The two species share a subset of their features: 10,816 paralogous genes. Each also has species-specific genes: 10,184 in the mouse dataset and 1,563 in the lizard dataset. As the data comes from different species there is no 1--1 correspondence between cells. However, the two species contain cells from similar cell types. Unlike the other single-cell dataset, there is a subset of the features (the paralogous genes) that have 1--1 correspondences across the two domains (domains are defined by species in this dataset).  

\paragraph{Baselines and hyperparameter tuning} 
We benchmark AGW's performance on single-cell alignment tasks against three algorithms: (1) COOT \cite{COOT}, (2) SCOT \cite{Demetci2020}, which is a Gromov-Wasserstein OT-based algorithm that uses k-nearest neighbor (kNN) graph distances as intra-domain distance matrices (this choice of distances has been shown to perform better than Euclidean distances, cosine distances by \cite{Demetci2020}), and bindSC \cite{bindSC}. Among these, bindSC is not an OT-based algorithm: It employs bi-order cannonical correlation analysis to perform alignment. We include it as a benchmark as it is the only existing single-cell alignment algorithm that can perform feature alignments (in addition to cell alignments) in most cases.

When methods share similar hyperparameters in their formulation (e.g. entropic regularization constant, $\epsilon$ for methods that employ OT), we use the same hyperparameter grid to perform their tuning. Otherwise, we refer to the publication and the code repository for each method to choose a hyperparameter range. For SCOT, we tune four hyperparameters: $k \in \{20, 30, \dots, 150\}$, the number of neighbors in the cell neighborhood graphs, $\epsilon \in \{5e-4, 3e-4, 1e-4, 7e-3, 5e-3, \dots, 1e-2 \}$, the entropic regularization coefficient for the optimal transport formulation. Similarly, for both COOT and AGW, we sweep $\epsilon_1, \epsilon_2 \in \{5e-4, 3e-4, 1e-4, 7e-3, 5e-3, \dots, 1e-2 \}$ for the coefficients of entropic regularization over the sample and feature alignments. We use the same intra-domain distance matrices in AGW as in SCOT (based on kNN graphs). For all OT-based methods, we perform barycentric projection to complete the alignment. 

For bindSC, we choose the couple coefficient that assigns weight to the initial gene activity matrix $\alpha \in \{0, 0.1, 0.2, \dots 0.9\}$ and the couple coefficient that assigns weight factor to multi-objective function $\lambda \in \{0.1, 0.2, \dots, 0.9\}$. Additionally, we choose the number of canonical vectors for the embedding space $K \in \{3, 4, 5, 10, 30, 32\}$.  For all methods, we report results with the best performing hyperparameter combinations. 

\paragraph{Evaluation Metrics} When evaluating cell alignments, we use a metric previously used by other single-cell multi-omic integration tools \cite{liu_et_al:LIPIcs:2019:11040, singh20,cao2020unsupervised, Demetci2020, Pamona, Demetci2022, bindSC} called ``fraction of samples closer than the true match'' (FOSCTTM). For this metric, we compute the Euclidean distances between a fixed sample point and all the data points in the other domain. Then, we use these distances to compute the fraction of samples that are closer to the fixed sample than its true match, and then average these values for all the samples in both domains. This metric measures alignment error, so the lower values correspond to higher quality alignments.

To assess feature alignment performance, we investigate the accuracy of feature correspondences recovered. We mainly use two real-world datasets for this task - CITE-seq, and the cross-species scRNA-seq datasets (results on SNARE-seq and scGEM datasets are qualitatively evaluated due to the lack of ground-truth information). For the CITE-seq dataset, we expect the feature correspondences to recover the relationship between the 25 antibodies and the genes that encode them. To investigate this, we simultaneously align the cells and features of the two modalities using the 25 antibodies and 25 genes in an unsupervised manner. We compute the percentage of 25 antibodies whose strongest correspondence is their encoding gene.

For the cross-species RNA-seq dataset, we expect alignments between (1) the cell-type annotations common to the mouse and lizard datasets, namely: excitatory neurons, inhibitory neurons, microglia, OPC (Oligodendrocyte precursor cells), oligodendrocytes, and endothelial cells and (2) between the paralogous genes. For this dataset, we generate cell-label matches by averaging the rows and columns of the cell-cell alignment matrix yielded by AGW based on these cell annotation labels. We compute the percentage of these six cell-type groups that match as their strongest correspondence. For feature alignments, we compute the percentage of the 10,816 shared genes that are assigned to their corresponding paralogous gene with their highest alignment probability. For this dataset, we consider providing supervision at increasing levels on both sample and feature alignments. For feature-level supervision, $20\%$ supervision means setting the alignment cost of $\sim 20\%$ of the genes with their paralogous pairs to $0$. For sample-level supervision, $20\%$ supervision corresponds to downscaling the alignment cost of $\sim 20\%$ of the mouse cells from the aforementioned seven cell-types with the $\sim 20\%$ of lizard cells from their corresponding cell-type by $\frac{1}{\textrm{\# lizard cells in the same cell-type}}$. 

\subsection{Heterogeneous domain adaptation experiments}
We evaluate AGW against GW and COOT on source-target pairs from the Caltech-Office dataset \cite{Saenko10}by considering all pairs between the three domains: Amazon (A), Caltech-$256$ (C), and Webcam (W), similarly to Redko \textit{et al}. We randomly choose 20 samples per class and perform adaptation from CaffeNet to GoogleNet and repeat it 10 times. We report the average performance of each method along with standard deviation. Differently than Redko \textit{et al}, we (1) unit normalize the dataset prior to alignment as we empirically found it to boost all methods' average performance compared to using unnormalized datasets, (2) use cosine distances when defining intra-domain distance matrices for GW and AGW, as we found them to perform better than Euclidean distances, and (3) report results after hyperparameter tuning methods for each pair of datasets. Specifically, for each pair of (A)-(C), (A)-(W) etc, we sweep a hyperparameter grid over 5 runs of random sampling, choose the best performing combination, and run 10 runs of random sampling to report results. For all methods, we consider their version with no entropic regularization (either on the sample-level alignments, feature-level alignments or both), along with various levels of regularization. For entropic regularization over sample alignments, we consider $\epsilon_1 \in [ 5e-4, 1e-3, 5e-3, 1e-2, 5e-2, 0.1] $ for in methods.  For entropic regularization over feature alignments in COOT and AGW, we consider $\epsilon_2 \in [ 5e-4, 1e-3, 5e-3, 1e-2, 5e-2, 0.1] $. As interpolation coefficient of AGW, we consider $\alpha \in [ 0.1, 0.2, ..., 0.9]$.

\section{SNARE-seq and scGEM feature alignments}\label{appendix:scFeat}
Although we do not have ground-truth information on the feature correspondences in the SNARE-seq and scGEM datasets, to ensure comprehensive results, we present the feature alignments obtained from AGW on these datasets in Figure \ref{fig:SI-feats}. Since we align 1000 genes and 3000 chromatin regions from the SNARE-seq datasets, it is not possible to present all the feature correspondences obtained. Instead, we show the four cell-type marker genes and their top correspondences from the accessible chromatin regions in Panel A. The scientists who originally generated this dataset used the expression status of these four genes when labeling cell types in this dataset. Panel B shows the feature coupling matrix yielded by AGW, which can be interpreted in the light of the information presented in Panel C and Panel D. We detail the significance of the inferred correspondences below.

\begin{figure}[h]
    \centering
    \includegraphics[width=1.0\linewidth]{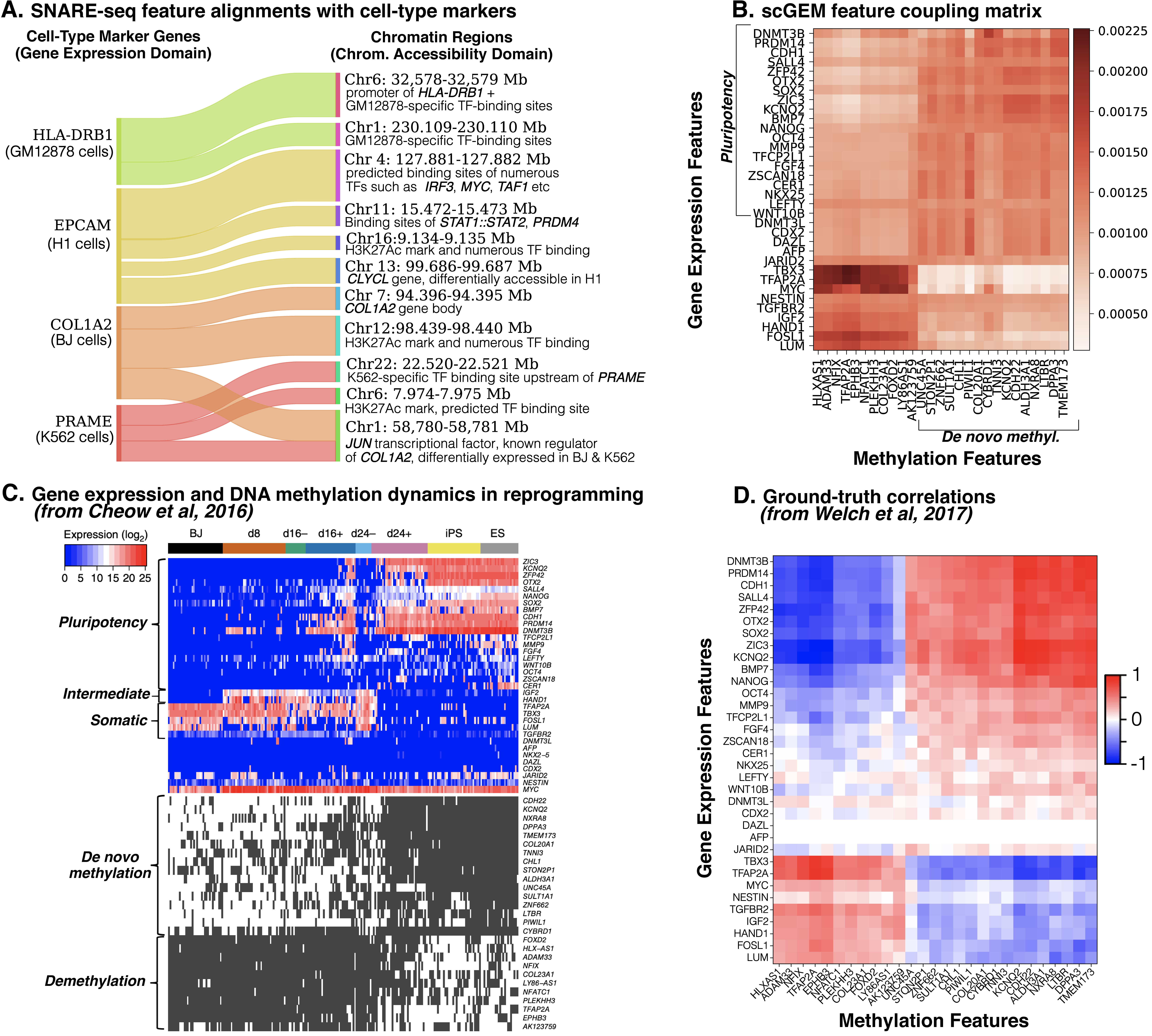}
    \caption{\textbf{AGW's feature alignments for A. SNARE-seq, and B. scGEM datasets}. The Sankey plot in Panel A presents the four cell-type marker genes and their top correspondences in the open chromatin regions. Panel B visualizes the feature coupling matrix for the scGEM dataset. Panel C is borrowed from the original publication introducing scGEM dataset, Cheow \textit{et al} \cite{cheow2016}, which shows how the genomic features in the two measurement domains (gene expression and DNA methylation) vary during cellular differentiation. Panel D is borrowed from Welch \textit{et al} that shows a heatmap of empirical correlations between the features of the two measurement domains, which we use for comparison with Panel B. \label{fig:SI-feats}}
\end{figure}
\paragraph{SNARE-seq feature correspondences: } Majority of the feature correspondences in Panel A are in agreement with either biologically validated or computationally predicted regulatory relationships. To validate these correspondences, we consult the biological annotations on the UCSC Genome Browser \cite{ucsc}, as well as gene regulatory information databases, such as GRNdb \cite{grndb} and RepMap Atlas of Regulatory Regions \cite{ReMap}.

Firstly, three of the alignments are between marker genes and their chromatin regions. These are (1) \textit{PRAME} and Chr22: 22.520-22.521 Mb region, which is a region upstream of the \textit{PRAME} gene body that is rich with predicted transcriptional factor (TF) binding sites according to the ``RepMap Atlas of Regulatory Regions'' \cite{ReMap} annotations on UCSC Genome Browser (Human hg38 annotations) \cite{ucsc}. Among the predicted TF bindings, many of them are K562-specific predictions, and some of these are known regulators of \textit{PRAME}, such as but not limited to \textit{E2F6}, \textit{HDAC2}, \textit{CTCF} (based on GRNdb database \cite{grndb} of TF-gene relationships). Additionally, (2) \textit{COL1A2} and (3) \textit{HLA-DRB1} also have recovered correspondences with their own chromosomal region, ``Chr7:94.395-94.396 Mb'' and ``Chr6:32,578-32,579 Mb'', respectively. We observe that \textit{COL1A2} and \textit{PRAME} are also additionally aligned with ``Chr1: 58,780 - 58,781 Mb'' region, which correspond to the gene body of \textit{JUN} transcriptional factor. Indeed, \textit{JUN} has been identified as one of the transcriptional factors differentially expressed in the K562 and BJ cells, but more strongly in the latter, according to the original publication that released this dataset \cite{SNAREseq}. GRNdb also identified \textit{JUN} to be one of the regulators of the \textit{COL1A2} gene. In addition to the chromosomal region of \textit{JUN}, \textit{PRAME} has another region abundant in predicted TF binding sites among its top correspondences: ``Chr6: 7.974-7.975 Mb''. This region is annotated with an H3K27Ac mark on the UCSC Genome Browser, which is a histone protein acetylation mark that is often found near gene regulatory elements on the genome \cite{ucsc}. Furthermore, this region contains multiple predicted binding sites of TFs GRNdb identifies as regulators of \textit{PRAME}, such as \textit{IRF1, HDAC2, HOXC6} and \text{POU2AF1}.  The \textit{HLA-DRB1} gene is also aligned with a chromosomal region rich in GM12878-specific predictions of TF bindings, such as \textit{IRF4}, \textit{IRF8}, \textit{ETV6}, and \textit{CREM}, which GRNdb lists as potential regulators of \textit{HLA-DRB1}. Lastly, even though we couldn't find a biological relationship reported between the \textit{CLYBL} gene and \textit{EPCAM} gene (marker gene for the H1 cell-line), the chromosomal region in \textit{CLYBL} body where AGW finds a correspondence with \textit{EPCAM} indeed appears to be differentially accessible in H1 cells in our dataset.  

\subsection{scGEM feature correspondences: } To interpret the feature coupling matrix we recover on the scGEM dataset, we consult the original publication that introduced this dataset \cite{cheow2016}. Figure \ref{fig:SI-feats} Panel C presenta a plot from this paper, which shows how the expression of genes that drive pluripotency during cell differentiation correlate or anti-correlate with the methylation of the genes in the ``DNA methylation'' domain. Based on the same pattern, Welch \textit{et al} generated the heatmap visualized in Panel D, which shows the underlying correlations between the expression and methylation levels of the genes in the two domains (i.e. gene expression and DNA methylation measurement domains). We observe that the feature coupling we receive from AGW (Panel B) resembles the structure in this ground-truth correlation matrix. Note that the features in the rows and columns of these matrices are ordered identically to aid with the visual comparison. Moreover, we see that the positive relationship between the expression profiles of pluripotency-driving genes and the methylation levels of associated genes is also recovered in this feature coupling matrix.

\subsection{Heterogeneous domain adaptation experiments}
We present the unsupervised and semi-supervised case with $t=3$ samples used for supervision in the main paper. Here, we additionally present the semi-supervised cases with $t=1$ and $t=5$ samples used for supervision in the table below:

\begin{table}[h]
\caption{\label{tabSI:hda1} \textbf{Heterogeneous domain adaptation results in the semi-supervised case with $t=1$ and $t=5$ samples used for supervision}. Best results are bolded. In AGW $(\alpha=0.5)$ columns, we also underline the results that outperform GW and COOT. In the ``AGW (Best $\alpha$)'' columns, the $\alpha$ values used are: $0.6, 0.2, 0.2, 0.9, 0.4, 0.5, 0.9, 0.3, 0.1$ top to bottom for the $t=1$ setting, and $0.3, 0.1, 0.7, 0.1, 0.5, 0.8, 0.9, 0.2, 0.9$ for the $t=5$ setting.}
\resizebox{.98\linewidth}{!}{
\begin{tabular}{lcccc|cccc}
\hline
                           & \multicolumn{4}{c|}{\textbf{Semi-supervised (t=1)}}                                                                                                                                         & \multicolumn{4}{c}{\textbf{Semi-supervised (t=5)}}                                                                                                                                         \\ \hline
                           & \textbf{COOT}   & \textbf{GW}             & \textbf{\begin{tabular}[c]{@{}c@{}}AGW \\ ($\alpha=0.5$)\end{tabular}} & \textbf{\begin{tabular}[c]{@{}c@{}}AGW\\ (Best $\alpha$)\end{tabular}} & \textbf{COOT}  & \textbf{GW}             & \textbf{\begin{tabular}[c]{@{}c@{}}AGW \\ ($\alpha=0.5$)\end{tabular}} & \textbf{\begin{tabular}[c]{@{}c@{}}AGW\\ (Best $\alpha$)\end{tabular}} \\ \hline
\textbf{A $\rightarrow$ A} & 87.1 $\pm$ 4.9  & 92.4 $\pm$ 1.8          & 90.5 $\pm$ 2.4                                                         & \textbf{93.1 $\pm$ 1.6}                                                & 93.6 $\pm$ 1.5 & 93.8 $\pm$ 2.1          & {\ul 94.7 $\pm$ 1.5}                                                   & \textbf{96.4 $\pm$ 1.3}                                                \\
\textbf{A $\rightarrow$ C} & 44.4 $\pm$ 4.9  & \textbf{90.6 $\pm$ 5.5} & 85.2 $\pm$ 6.0                                                         & 90.2 $\pm$ 5.1                                                         & 66.1 $\pm$ 3.8 & 91.9 $\pm$ 2.2          & {\ul 92.2 $\pm$ 1.8}                                                   & \textbf{96.4 $\pm$ 1.7}                                                \\
\textbf{A $\rightarrow$ W} & 54.5 $\pm$ 9.8  & 85.7 $\pm$ 3.9          & {\ul 87.9 $\pm$ 6.4}                                                   & \textbf{90.3 $\pm$ 3.5}                                                & 75.7 $\pm$ 3.5 & \textbf{93.3 $\pm$ 1.2} & 92.0 $\pm$ 2.4                                                         & 93.1 $\pm$ 2.9                                                         \\
\textbf{C $\rightarrow$ A} & 74.5 $\pm$ 3.1  & 77.6 $\pm$ 8.6          & 71.0 $\pm$ 9.9                                                         & \textbf{78.8 $\pm$ 7.7}                                                & 85.2 $\pm$ 2.1 & 85.2 $\pm$ 3.6          & 82.7 $\pm$ 2.2                                                         & \textbf{86.2 $\pm$ 2.3}                                                \\
\textbf{C $\rightarrow$ C} & 45.4 $\pm$ 12.2 & 82.8 $\pm$ 3.0          & {\ul 83.4 $\pm$ 3.6}                                                   & \textbf{84.2 $\pm$ 2.3}                                                & 64.7 $\pm$ 7.2 & 84.5 $\pm$ 2.7          & {\ul 86.6 $\pm$ 1.9}                                                   & \textbf{86.6 $\pm$ 1.9}                                                \\
\textbf{C $\rightarrow$ W} & 40.8 $\pm$ 12.2 & 73.3 $\pm$ 6.9          & {\ul \textbf{77.3 $\pm$  4.2}}                                         & \textbf{77.3 $\pm$ 4.2}                                                & 67.0 $\pm$ 7.5 & 81.9 $\pm$ 3.6          & 80.3 $\pm$ 2.1                                                         & \textbf{83.9 $\pm$ 1.4}                                                \\
\textbf{W $\rightarrow$ A} & 86.9 $\pm$ 2.3  & 91.9 $\pm$ 7.9          & 87.7 $\pm$ 7.3                                                         & \textbf{92.8 $\pm$ 3.8}                                                & 96.3 $\pm$ 1.9 & 96.6 $\pm$ 1.1          & 95.9 $\pm$ 1.5                                                         & \textbf{96.7 $\pm$ 1.1}                                                \\
\textbf{W $\rightarrow$ C} & 39.7 $\pm$ 3.6  & 81.9 $\pm$ 9.6          & {\ul 85.5 $\pm$ 5.9}                                                   & \textbf{90.3 $\pm$ 3.7}                                                & 60.5 $\pm$ 5.3 & 94.5 $\pm$ 3.1          & 91.9 $\pm$ 1.7                                                         & \textbf{95.7 $\pm$ 2.1}                                                \\
\textbf{W $\rightarrow$ W} & 76.1 $\pm$ 13.5 & 97.2 $\pm$ 1.4          & {\ul 98.2 $\pm$ 0.9}                                                   & \textbf{98.5 $\pm$ 0.8}                                                & 90.3 $\pm$ 1.9 & 98.4 $\pm$ 0.9          & {\ul 98.5 $\pm$ 0.9}                                                   & \textbf{98.7 $\pm$ 0.9}                                                \\ \hline
\textbf{Mean} & 61.0  &  85.9  & 85.2 & \textbf{88.4} & 77.7 &  91.1  & 90.5 & \textbf{93.7} \\ \hline                                                                       
\end{tabular}
}
\end{table}

\vspace{0.5cm}
\begin{table}[h]
\caption{\label{tabSI:hda2} \textbf{Heterogeneous domain adaptation results for the unsupervised setting and the semi-supervised setting with $t=3$ samples used for supervision}. Best results are bolded.  In AGW $(\alpha=0.5)$ columns, we also underline the results that outperform GW and COOT. In the ``AGW (Best $\alpha$)'' column, the $\alpha$ values used are: $0.6, 0.9, 0.7, 0.9, 0.3, 0.8, 0.7, 0.2, 0.6$ top to bottom for the unsupervised setting, and $0.2, 0.1, 0.2, 0.7, 0.2, 0.9, 0.8, 0.9, 0.4$ for the semi-supervised setting.}
\resizebox{.98\linewidth}{!}{
\begin{tabular}{@{}lcccc|cccc@{}}
\toprule
                           & \multicolumn{4}{c|}{\textbf{Unsupervised}}                                                                                                                                          & \multicolumn{4}{c}{\textbf{Semi-supervised (t=3)}}                                                                                                                                                     \\ \midrule
                           & \textbf{COOT}   & \textbf{GW}     & \textbf{\begin{tabular}[c]{@{}c@{}}AGW \\ ($\alpha$=0.5)\end{tabular}} & \textbf{\begin{tabular}[c]{@{}c@{}}AGW\\ (Best $\alpha$)\end{tabular}} & \textbf{COOT}  & \textbf{GW}                   & \textbf{\begin{tabular}[c]{@{}c@{}}AGW \\ ($\alpha$=0.5)\end{tabular}} & \textbf{\begin{tabular}[c]{@{}c@{}}AGW\\ (Best $\alpha$)\end{tabular}} \\ \midrule
\textbf{A $\rightarrow$ A} & 50.3 $\pm$ 15.9 & 86.2 $\pm$ 2.3  & \underline{90.5 $\pm$ 2.4}                                                   & \textbf{93.1 $\pm$ 1.6}                                 & 91.1 $\pm$ 2.0 & 93.2 $\pm$ 0.9                & \underline{93.8 $\pm$ 1.3}             & \textbf{96.0 $\pm$ 0.8}                                \\
\textbf{A $\rightarrow$ C} & 35.0 $\pm$ 6.4  & 64.1 $\pm$ 6.2  & \underline{68.2 $\pm$ 7.4}                                                   & \textbf{68.3 $\pm$ 14.1}                                & 59.7 $\pm$ 3.6 & 92.8 $\pm$ 2.1                & 90.7 $\pm$ 1.9                                                         & \textbf{93.5 $\pm$ 1.8}                                 \\
\textbf{A $\rightarrow$ W} & 39.8 $\pm$ 14.5 & 79.6 $\pm$ 11.1 & 75.5 $\pm$ 3.1                                                         & \textbf{79.8 $\pm$ 3.5}                                 & 72.6 $\pm$ 4.4 & 91.6 $\pm$ 1.8                & 91.4 $\pm$ 1.1                                                         & \textbf{93.8 $\pm$ 0.7}                                 \\
\textbf{C $\rightarrow$ A} & 40.8 $\pm$ 15.8 & 53.0 $\pm$ 13.2 & 48.5 $\pm$ 6.9                                                         & \textbf{55.4 $\pm$ 7.1}                                 & 83.1 $\pm$ 5.1 & 81.2 $\pm$ 1.2                & \underline{84.3 $\pm$ 1.6}                                                   & \textbf{85.6 $\pm$ 1.2}                                 \\
\textbf{C $\rightarrow$ C} & 33.4 $\pm$ 10.7 & \textbf{81.9 $\pm$ 30.5} & 68.5 $\pm$ 5.5                                                         & 76.4 $\pm$ 5.6                                          & 59.3 $\pm$ 8.4 & 85.3 $\pm$ 2.8                & 83.4 $\pm$ 2.3                                                         & \textbf{86.5 $\pm$ 2.1}                                 \\
\textbf{C $\rightarrow$ W} & 37.5 $\pm$ 10.4 & 53.5 $\pm$ 15.9 & \underline{56.6 $\pm$ 7.6}                                                   & \textbf{57.7 $\pm$ 14.3}                                & 64.6 $\pm$ 6.2 & 79.7 $\pm$ 2.5                & \underline{81.3 $\pm$ 4.3}                                                         & \textbf{83.2 $\pm$ 2.4}                                 \\
\textbf{W $\rightarrow$ A} & 44.3 $\pm$ 14.0 & 50.4 $\pm$ 22.1 & \underline{52.1 $\pm$ 3.8}                                                   & \textbf{60.1 $\pm$ 9.1}                                 & 94.3 $\pm$ 2.2 & 93.4 $\pm$ 5.2                & 92.3 $\pm$ 1.5                                                         & \textbf{97.1 $\pm$ 0.8}                                 \\
\textbf{W $\rightarrow$ C} & 27.4 $\pm$ 10.2 & 54.3 $\pm$ 14.7 & 53.6 $\pm$ 17.3                                                        & \textbf{60.9 $\pm$ 13.3}                                & 55.0 $\pm$ 7.1 & 90.9 $\pm$ 3.5                & 90.9 $\pm$ 2.0                                                         & \textbf{94.7 $\pm$ 1.1}                                 \\
\textbf{W $\rightarrow$ W} & 57.9 $\pm$ 13.4 & 92.5 $\pm$ 2.6  & 90.3 $\pm$ 5.4                                                         & \textbf{97.2 $\pm$ 0.9}                                 & 87.4 $\pm$ 4.4 & 97.4 $\pm$ 2.6\ & \underline{98.5 $\pm$ 0.7}                                                   & \textbf{98.7 $\pm$ 0.5}                                 \\ \hline
\textbf{Mean}    & 40.7 &   60.6 &   67.1 & 72.1 & 74.1 &  89.5  &   89.6 &  92.1  \\ \hline                                            

\end{tabular}
}
\end{table}

\newpage
\section{Empirical runtime comparison with COOT and GW}

For timing, we run all algorithms  on an Intel Xeon e5-2670 CPU with 16GB memory. For GW, we use the implementation in Python's POT library with its NumPy backend. We use the COOT implementation on \url{https://github.com/ievred/COOT}. Note that the strength of entropic regularization picked influences the runtime. We report hyperparameters that were picked for each case in the experiment replication scripts on \url{https://github.com/pinardemetci/AGW}.

\begin{table}[h]
\centering
\caption{\label{tab:mnistTime} \textbf{Runtime (in seconds) of each algorithm on the MNIST dataset in their best hyperparameter setting}. Timing does not include hyperparameter tuning time. We report the average of 5 runs with standard deviations.}
\begin{tabular}{@{}ccc@{}}
\toprule
\textbf{\begin{tabular}[c]{@{}c@{}}COOT\\ \end{tabular}} & \textbf{\begin{tabular}[c]{@{}c@{}}GW\\ \end{tabular}} & \textbf{\begin{tabular}[c]{@{}c@{}}AGW\\\end{tabular}} \\ \midrule
3.15 $\pm$ 1.3                                                                               & 51.35 $\pm$ 7.73                                                          & 11.71 $\pm$ 2.86              \\ \hline                                                                         
\end{tabular}
\end{table}

\vspace{0.5cm}
\begin{table}[h]
\caption{
\label{tab:hdaRuntime1} \textbf{Runtime (in seconds) of each algorithm in the heterogeneous domain adaptation experiments}. Timing does not include hyperparameter tuning time; it only includes 10 runs of the best hyperparameter combination for each algorithm.}
\resizebox{.98\linewidth}{!}{
\begin{tabular}{lcccc|cccc}
\hline
                           & \multicolumn{4}{c|}{\textbf{Semi-supervised (t=1)}}                                                                                                                           & \multicolumn{4}{c}{\textbf{Semi-supervised (t=5)}}                                                                                                                            \\ \hline
                           & \textbf{COOT} & \textbf{GW} & \textbf{\begin{tabular}[c]{@{}c@{}}AGW \\ ($\alpha=0.5$)\end{tabular}} & \textbf{\begin{tabular}[c]{@{}c@{}}AGW\\ (Best $\alpha$)\end{tabular}} & \textbf{COOT} & \textbf{GW} & \textbf{\begin{tabular}[c]{@{}c@{}}AGW \\ ($\alpha=0.5$)\end{tabular}} & \textbf{\begin{tabular}[c]{@{}c@{}}AGW\\ (Best $\alpha$)\end{tabular}} \\ \hline
\textbf{A $\rightarrow$ A} & 4.76s         & 2.37s       & 4.55s                                                                  & 7.50s                                                                  & 2.38s         & 1.70s       & 1.67s                                                                  & 2.07s                                                                  \\
\textbf{A $\rightarrow$ C} & 1.22s         & 3.05s       & 1.85s                                                                  & 3.83s                                                                  & 1.17s         & 2.49s       & 2.66s                                                                  & 1.81s                                                                  \\
\textbf{A $\rightarrow$ W} & 3.23s         & 4.65s       & 1.77s                                                                  & 3.84s                                                                  & 1.89s         & 3.04s       & 1.15s                                                                  & 1.77s                                                                  \\
\textbf{C $\rightarrow$ A} & 7.49s         & 3.85s       & 2.86s                                                                  & 4.13s                                                                  & 3.11s         & 2.96s       & 2.53s                                                                  & 1.78s                                                                  \\
\textbf{C $\rightarrow$ C} & 2.00s         & 1.50s       & 1.71s                                                                  & 1.99s                                                                  & 1.95s         & 1.19s       & 1.83s                                                                  & 1.83s                                                                  \\
\textbf{C $\rightarrow$ W} & 1.71s         & 4.74s       & 3.38s                                                                  & 3.38s                                                                  & 1.05s         & 2.01s       & 1.50s                                                                  & 1.76s                                                                  \\
\textbf{W $\rightarrow$ A} & 5.75s         & 9.17s       & 5.19s                                                                  & 3.51s                                                                  & 4.05s         & 2.95s       & 0.74s                                                                  & 1.42s                                                                  \\
\textbf{W $\rightarrow$ C} & 0.67s         & 9.05s       & 3.54s                                                                  & 3.15s                                                                  & 0.91s         & 3.82s       & 2.67s                                                                  & 1.70s                                                                  \\
\textbf{W $\rightarrow$ W} & 0.58s         & 3.73s       & 2.82s                                                                  & 3.79s                                                                  & 2.00s         & 3.77s       & 1.53s                                                                  & 2.27s      \\ \hline                                                           
\end{tabular}
}
\end{table}
\newpage
\bibliographystyle{unsrt}
\bibliography{main}
\end{document}